\documentclass{article}
\usepackage{ijcai21}
\usepackage{times}
\usepackage{soul}
\usepackage[hidelinks]{hyperref}
\usepackage[utf8]{inputenc}
\usepackage[small]{caption}
\usepackage{graphicx}
\usepackage{amsmath}
\usepackage{amsthm}
\usepackage{booktabs}
\usepackage{algorithmic}
\usepackage{amsfonts}
\usepackage{bm}
\usepackage[ruled]{algorithm2e} % For algorithms
\usepackage{multicol}

\usepackage{algorithmic}
\usepackage{color}
\usepackage{enumitem}
\usepackage{multicol}
\usepackage{multirow}
\usepackage[table]{xcolor}
\usepackage{xurl}
\newcommand\ChangeRT[1]{\noalign{\hrule height #1}}
\usepackage{flushend}

\newcommand{\ie}{\emph{i.e., }}

\newcommand{\eg}{\emph{e.g., }}

\newtheorem{theorem}{Theorem}

\title{Data-Efficient Reinforcement Learning for Malaria Control}

\author{
Lixin Zou$^1$
\and
Long Xia$^1$\and
Linfang Hou$^2$\and 
Xiangyu Zhao$^{3}$\And
Dawei Yin$^1$
\affiliations
$^1$Baidu Inc., $^2$ JD.com\\
$^3$Michigan State University\\
\emails
\{zoulixin15, long.phil.xia, houlinfang09\}@gmail.com,
zhaoxi35@msu.edu,
yindawei@acm.org
}

\begin{document}

\maketitle

\begin{abstract}
Sequential decision-making under cost-sensitive tasks is prohibitively daunting, especially for the problem that has a significant impact on people's daily lives, such as malaria control, treatment recommendation. The main challenge faced by policymakers is to learn a policy from scratch by interacting with a complex environment in a few trials. This work introduces a practical, data-efficient policy learning method, named Variance-Bonus Monte Carlo Tree Search~(VB-MCTS), which can copy with very little data and facilitate learning from scratch in only a few trials. Specifically, the solution is a model-based reinforcement learning method. To avoid model bias, we apply Gaussian Process~(GP) regression to estimate the transitions explicitly. With the GP world model, we propose a variance-bonus reward to measure the uncertainty about the world. Adding the reward to the planning with MCTS can result in more efficient and effective exploration. Furthermore, the derived polynomial sample complexity indicates that VB-MCTS is sample efficient. Finally, outstanding performance on a competitive world-level RL competition and extensive experimental results verify its advantage over the state-of-the-art on the challenging malaria control task. 
\end{abstract}

\section{Introduction}

Malaria is a mosquito-borne disease that continues to pose a heavy burden on South Sahara Africa~(SSA)~\cite{moran2007malaria}. Recently, there has been significant progress in improving treatment efficiency and reducing the mortality rate of malaria. Unfortunately, due to financial constraints, the policymakers face the challenge of ensuring continued success in disease control with insufficient resources. To make intelligent decisions, learning control policies over the years have been formulated as Reinforcement Learning~(RL) problems~\cite{bent2018novel}. Nevertheless, applying RL to malaria control seems to be a tricky issue since RL usually requires numerous trial-and-error searches to learn from scratch. Unlike simulation-based games, \eg Atari games~\cite{mnih2013playing} and game GO~\cite{silver2017mastering}, the endless intervention trial is unacceptable to regions over the years since the actual cost of life and money is enormous. Hence, as in many human-in-loop systems~\cite{zou2019longterm,zou2020neural,zou2020pseudo}, it is too expensive to apply RL to learn malaria intervention policy from scratch directly.

Therefore, to reduce the heavy burden of malaria in SSA, it is urgent to improve the data efficiency of policy learning. In~\cite{bent2018novel}, novel exploration techniques, such as Genetic Algorithm~\cite{holland1992genetic}, Batch Policy Gradient~\cite{sutton2000policy} and Upper/Lower Confidence Bound~\cite{auer2010ucb}, have been firstly applied to learn malaria control policies from scratch. However, these solutions are introduced under the Stochastic Multi-Armed Bandit~(SMAB) setting, which myopically ignores the delayed impact of the interventions in the future and might result in serious problems. For example, the large-scale use of spraying may lead to mosquito resistance and bring about the uncontrolled spread of malaria in the coming years. Hence, it requires us to optimize disease control policies in the long run, which is far more challenging than the exploration in SMAB.

Considering the long-term effects, the finite horizon continuous-space Markov Decision Process is employed to model the disease control in this work. Under this setting, we propose a framework named Variance-Bonus Monte Carlo Tree Search~(VB-MCTS) for data-efficient policy searching, illustrated in Figure~\ref{fig:solution}. Particularly, it is a model-based training framework, which iterates between updating the world model and collecting data. In model training, Gaussian Process~(GP) is used to approximate the state transition function with collected rollouts. As a non-parametric probabilistic model, GP can avoid the model bias and explicitly model the uncertainty about the transitions, \ie~the variance of the state. In data collection, we propose to employ MCTS for generating the policy with the mean MDP plus variance-bonus reward. The variance-bonus reward can decrease the uncertainty at the state-action pairs with high potential reward by explicitly motivating the agent to sample the state-actions with the highest upper-bounded reward. Furthermore, the sample complexity of the proposed method indicates that it is a PAC optimal exploration solution for malaria control. Finally, to verify the effectiveness of our policy search solution, extensive experiments are conducted on the malaria control simulators\footnote{\url{https://github.com/IBM/ushiriki-policy-engine-library}}~\cite{bent2018novel}, which are gym-like\footnote{\url{https://gym.openai.com}} environments for KDD Cup 2019~\cite{zhou2020kdd}. The outstanding performance on the competition and extensive experimental results demonstrated that our approach could achieve unprecedented data efficiency on malaria control compared to the state-of-the-art methods.

\begin{figure}
\centering
\includegraphics[width=2.7in]{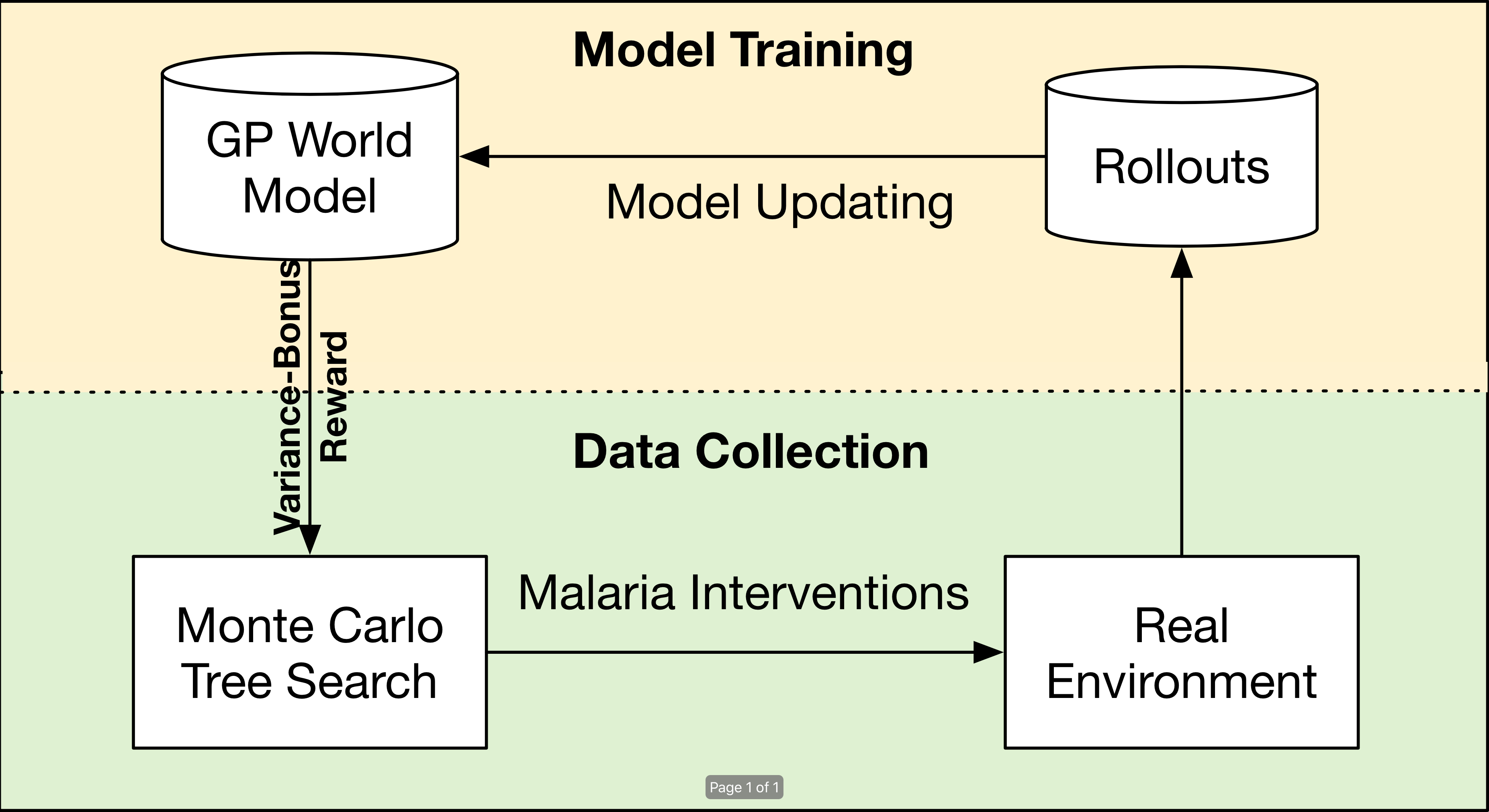}
    \caption{Overview of the proposed learning system on Malaria Control. The system 
    alternates model training and data collection.}
    \label{fig:solution} 
\end{figure}

Our main contributions are: \textbf{(1)} We propose a highly data-efficient learning framework for malaria control. Under the framework, the policymakers can successfully learn control policies from scratch within 20 rollouts. \textbf{(2)} We derive the sample complexity of the proposed method and verify that VB-MCTS is an efficient PAC-MDP algorithm. \textbf{(3)} Extensive experiments conducted on malaria control demonstrate that our solution can outperform the state-of-the-art methods.

\section{Related Work}
As a highly pathogenic disease, malaria has been widely studied from the perspective of predicting the disease spread, diagnosis, and personalized care planning. However, rarely work focuses on applying RL to learn the cost-effectiveness intervention strategies, which plays a crucial role in controlling the spread of malaria~\cite{moran2007malaria}. In~\cite{bent2018novel}, malaria control has firstly been formulated as a stochastic multi-armed bandit (SMAB) problem and solved with novel exploration techniques. Nevertheless, SMAB based solutions only myopically maximize instant rewards, and the ignorance of delayed influences might result in disease outbreaks in the future. Therefore, in this work, a comprehensive solution has been proposed to facilitate policy learning in a few trials under the setting of finite-horizon MDP. 

Another topic is data-efficient RL. To increase the data efficiency, we are required to extract more information from available trials~\cite{deisenroth2011pilco}, which involves utilizing the samples in the most efficient way (\eg exploitation) and choosing the samples with more information (\eg exploration). Generally, for exploitation, model-based methods~\cite{ha2018recurrent,kamthe2017data} are more sample efficient but require more computation time for the planning. Model-free~\cite{szita2006learning,krause2016cma,van2009theoretical} methods are generally computationally light and can be applied without a planner, but need (sometimes exponentially) more samples, and are usually not efficient PAC-MDP algorithm~\cite{strehl2009reinforcement}. For exploration, there are two options: \textbf{(1)} Bayesian approaches, considering a distribution over possible models and acting to maximize expected reward; unluckily, it is intractable for all but very restricted cases, such as the linear policy assumption in PILCO~\cite{deisenroth2011pilco}. \textbf{(2)} intrinsically motivated exploration, implicitly negotiating the exploration/exploitation dilemma by always exploiting a modified reward for directly accomplishing exploration. However, on the one hand, the vast majority of papers only address the discrete state case, providing incremental improvements on the complexity bounds, such as ~MMDP-RB~\cite{sorg2012variance}, metric-E$^3$~\cite{kakade2003exploration}, and  BED~\cite{kolter2009near}. On the other hand, for more realistic continuous state space MDP, over-exploration has been introduced for achieving polynomial sample complexity in many work, such as KWIK~\cite{li2011knows}, and GP-Rmax~\cite{grande2014sample}. These methods will explore all regions equally until the reward function is highly accurate everywhere. By drawing on the strength of existing methods, our solution is a model-based RL framework, which efficiently plans with MCTS and trade-off exploitation and exploration by exploiting a variance-bonus reward.

\section{Proposed Method: VB-MCTS}
\subsection{Malaria Control as MDP}
Finding an optimal malaria control policy can be posted as a reinforcement learning task by illustrating it as a Markov Decision Process~(MDP).
Specifically, we formulate the task as a finite-horizon MDP, defined by the tuple $\langle\mathcal{S},\mathcal{A},P,R,\gamma\rangle$ with $\mathcal{S}$ as the potential infinite state space, $\mathcal{A}$ as the finite set of actions, $P:\mathcal{S}\times\mathcal{A} \rightarrow \mathcal{S}$ as the deterministic transition function,  $R:\mathcal{S}\rightarrow \mathbb{R}$ as the reward function, and $\gamma\in (0,1]$ as the discount factor.
In this case, we face the challenge of developing an efficient policy for a population over a 5 year intervention time frame.
As shown in Figure \ref{fig:mdp}, the corresponding components in malaria control are defined as,

\begin{figure}
\centering
\includegraphics[width=2.7in]{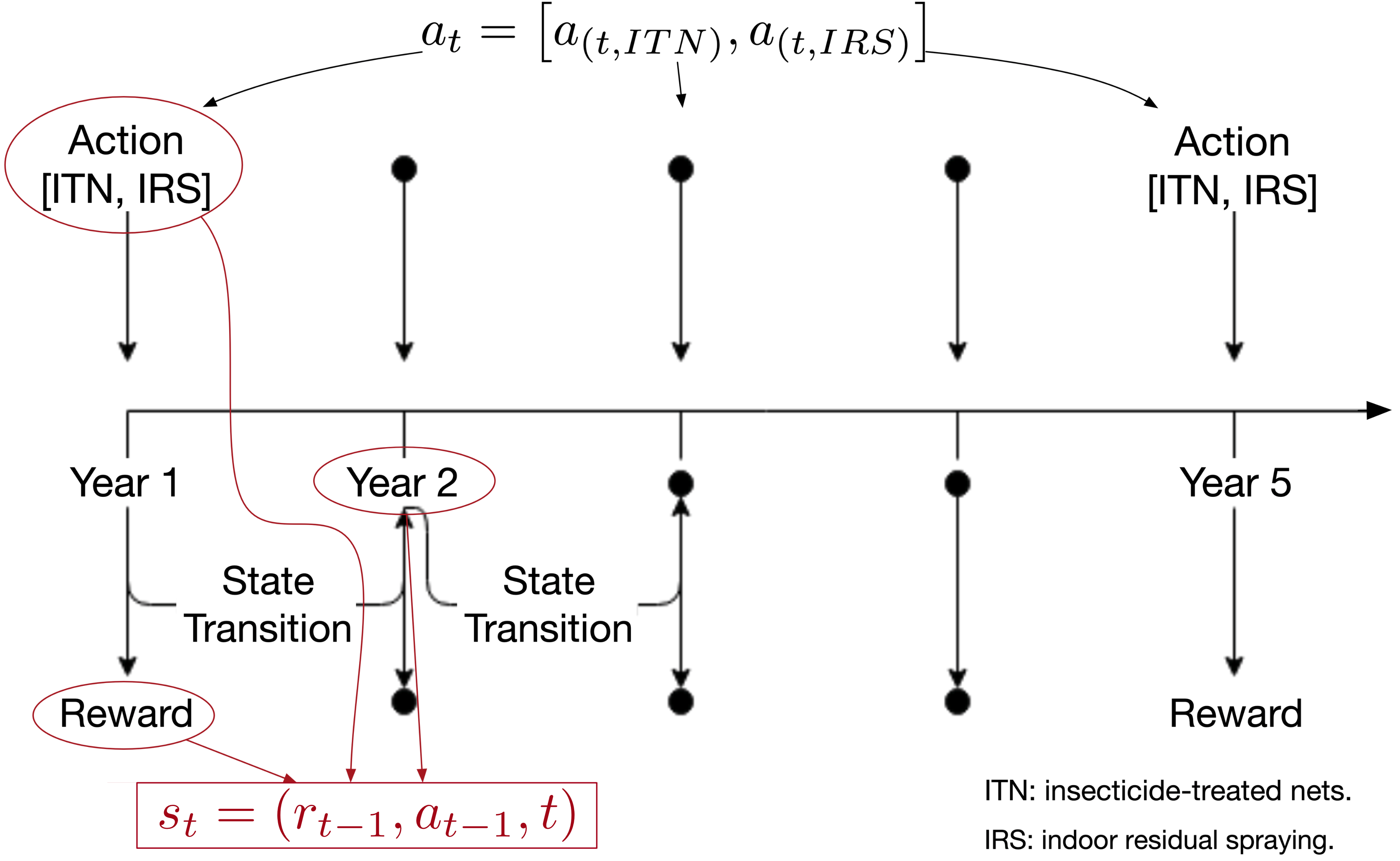}
    \caption{Malaria intervention as MDP.}
    \label{fig:mdp} 
\end{figure}

\paragraph{Action} The actions are the available means of interventions, including the mass-distribution of long-lasting insecticide-treated nets~(ITNs) and indoor residual spraying~(IRS) with pyrethroids in SSA~\cite{stuckey2014modeling}. In this work, the action space $\mathcal{A}$ is constructed through $a_i \in \mathcal{A} = \{(a_{ITN}, a_{IRS})\}$ with $a_{ITN},a_{IRS} \in (0,1]$, which represent the population coverage for ITNs and IRS of a specific area. Without significantly affecting performance, we discrete the action space with an accuracy of 0.1 for simplicity.

\paragraph{Reward} The reward is a scalar $r_t=R(s_{t+1})\in [R_{\min},R_{\max}]$, associated with next state $s_{t+1}$. In malaria control, it is determined through an economic cost-effectiveness analysis. In~\cite{bent2018novel}, an overview of the reward calculation is specified. Without loss of generality, the reward function $R(\ast)$ is assumed known to us since an MDP with unknown rewards and unknown transitions can be represented as an MDP with known rewards and unknown transitions by adding additional states to the system.
 
 \paragraph{State} The state contains important observations for decision making in every time step. In malaria control, it includes the number of life with disability, life expectancy, life lost, and treatment expenses~\cite{bent2018novel}. We set the state in the form $s_t = (r_{t-1},a_{t-1},t)$, including current reward $r_{t-1}$, previous action $a_{t-1}$ and the current intervention timestamp, which covers the most crucial and useful observations for malaria control, as shown in Figure \ref{fig:mdp}. For the start state $s_1$, the reward and action are initialized with 0. 
%    For simplicity of notation, we will assume that the reward is known, but this does not sacrifice generality, since an MDP with unknown (bounded) reward and unknown transitions can be represented as an MDP with known reward by adding additional states to the system.

Let $\pi:\mathcal{S}\rightarrow\mathcal{A}$ denote a deterministic mapping from states to actions, and let $V_{\pi}(s_t)=\mathbb{E}_{a_{i}\sim\pi(s_i)}\left[\sum_{i=t}^{T} \gamma^{i-t} r_i\right]$ denote the expected discounted reward by following policy $\pi$ in state $s_t$. The objective is to find a deterministic policy $\pi^\ast$ that maximizes the expected return at state $s_1$ as 
\begin{eqnarray}\nonumber
    \pi^\ast = {\arg\max}_{\pi\in \Pi} V_{\pi}(s_1).
\end{eqnarray}

\subsection{Model-based Indirect Policy Search}
In the following, we detail the key components of the proposed framework VB-MCTS, including the world model, the planner, and variance-bonus reward with its sample complexity.

\paragraph{World Model Learning}
The probabilistic world model is implemented as a GP, where we use a predefined feature mapping $x_t = \phi(s_t,a_t) \in \mathbb{R}^{m}$ as training input and the target state $s_{t+1}$ as the training target. The GP yields one-step predictions
\begin{eqnarray}\nonumber
    p(s_{t+1}|s_t,a_t) &=& \mathcal{N} (s_{t+1}|\mu_t,\sigma^2_t)\\ \nonumber
    \mu_t &=& \bm{k}_{\ast}^\top(K+\omega^2_n I)^{-1}\bm{y},\\ \nonumber
    \sigma^2_t &=& k(x_{t},x_{t}) - \bm{k}_{\ast}^\top(K+\omega^2_n I)^{-1}\bm{k}_{\ast},
\end{eqnarray}
where $k$ is the kernel function, $\bm{k}_{\ast} \equiv k(\bm{X},x_{t})$ denotes the vector of covariances between the test point and all training points with $\bm{X}=[x_1,\cdots,x_n] $, and $\bm{y}=[s_1,\dots,s_n]^\top$ is the corresponding training targets. $\omega_n$ is the noise variance. $K$ is the Gram matrix with entries $K_{ij} = k(x_i,x_j)$.  

Throughout this paper, we consider a prior mean function $m \equiv 0$ and a squared exponential~(SE) kernel with automatic relevance determination. The SE covariance function is defined as
\begin{eqnarray}\nonumber
     k(x,x') = \alpha^2 \exp (-\frac{1}{2}(x-x')^\top \mathbf{\Lambda}^{-1} (x-x')),
\end{eqnarray}
where $\alpha^2$ is the variance of state transition and $\bm{\Lambda} \equiv \operatorname{diag}\left(\left[l_{1}^{2}, \ldots, l_{m}^{2}\right]\right)$. The characteristic length-scale $l_i$ controls the importance of $i$-th feature. Given $n$ training inputs $\bm{X}$ and the corresponding targets $\bm{y}$, the posterior hyper-parameters of GP~(length-scales $l_i$ and signal variance $\alpha^2$) are determined through evidence maximization technique~\cite{williams2006gaussian}. 

\paragraph{Exploration with Variance-Bonus Reward}
%  With a unbiased probabilistic world model, a very elegant solution to the exploration/exploitation problem is explicitly representing the uncertainty over MDPs by maintaining a belief state, and choosing the actions based not only on how they will affect the next state of the system, but also based on how they will affect the next belief state; and, since a better knowledge of the MDP will typically lead to greater future reward, the Bayesian policy will very naturally trade off between exploring the system to gain more knowledge, and exploiting its current knowledge of the system~\cite{kolter2009near,sorg2012variance}. Unfortunately, it is typically not possible to compute the Bayesian policy exactly.

The algorithm we propose is itself very straightforward and similar to many previously proposed exploration
heuristics~\cite{kolter2009near,srinivas2009gaussian,sorg2012variance,grande2014computationally}. We call the algorithm Variance-Bonus Reward, since it chooses action according to the current mean estimation of the reward plus an additional variance-based reward bonus for state-actions that have not been well explored as
{\small
\begin{eqnarray} \nonumber
&& \tilde{R}(s_{t+1}|s_t,a_t) =
    R(s_{t+1})|_{s_{t+1} = \mathbb{E}_{s_{t+1}\sim GP(s_t,a_t)}[s_{t+1}]} + \\ \nonumber
&& \beta_1 \text{Var}_{s_{t+1}\sim GP(s_t,a_t)}[R(s_{t+1})]
    +	 \beta_2 \text{Var}_{s_{t+1}\sim GP(s_t,a_t)}[s_{t+1}],
\end{eqnarray}}
where $\beta_1$ and $\beta_2$ are the parameters that trade-off the balance of exploitation and exploration. $\text{Var}_{s_{t+1}\sim GP(s_t,a_t)}[R(s_{t+1})]$ and $\text{Var}_{s_{t+1}\sim GP(s_t,a_t)}[s_{t+1}]$ are the predicted variances for state and reward. The variance of reward can be exactly computed following the law of iterated variances.

\paragraph{Planning with Mean MDP + Reward Bonus}
MCTS is a strikingly successful planning algorithm~\cite{silver2017mastering}, which can find out the optimal solution with enough computation resources. Since disease control is not a real-time task, we propose to apply MCTS~(Figure~\ref{fig:solution}) as the planner for generating the policy to maximize the variance-bonus reward. In the executing process, the MCTS planner incrementally builds an asymmetric search tree guided to the most promising direction by a tree policy. This process usually consists of four phases --- \textbf{selection}, \textbf{expansion}, \textbf{evaluation}, and \textbf{backup}~(as shown in Figure~\ref{fig:mcts}).

Specifically, each edge $(s,a)$ of the search tree stores an average action value $\tilde{Q}(s, a)$ and visit count $N(s, a)$. In the \textbf{selection} phase, starting from the root state, the tree is traversed by simulation~(that is, descending the tree with the mean prediction of states without backup). At each time step $t$ of each simulation, an action $a_t$ is selected from state $s_t$ 
    \begin{eqnarray*}
        a_{t} = {\arg\max}_{a\in \mathcal{A}} ( \tilde{Q}(s_t,a)+ \frac{c_{puct}}{|\mathcal{A}|}\frac{\sqrt{\sum_{b\in \mathcal{A}}N(s_t,b)}}{1+N(s_t,a)} ),
    \end{eqnarray*} 
so as maximize action value plus a bonus that decays with repeated visits to encourage exploration for tree search. Here, $c_{puct}$ is the constant determining the level of exploration. When the traversal reaches a leaf node $s_L$ at step $L$, the leaf node may be \textbf{expanded} with each edge initialized as $\tilde{Q}(s_L,a) = 0$, $N(s_L,a) = 0$, and the corresponding leaf nodes are initialized with the mean GP prediction $s_{L+1}=\mathbb{E}[s_{L+1}|s_L,a]$. Then, the leaf node is \textbf{evaluated} by the average outcome $v(s_L)$ of rollouts, which played out until terminal step $T$ using the fast rollout policies, such as random policy and greedy policy. At the end of the simulation, the action values and visit counts of all traversed edges are updated, \ie \textbf{backup}. Each edge accumulates the visit counts and means the evaluation of all simulations passing through that edge as
\begin{eqnarray*}
      \tilde{Q}(s_j,a_j) &\leftarrow & \frac{N(s_j,a_j)\times \tilde{Q}({s}_j,{a}_j)+v({{s}_j})}{N(s_j,a_j)+1},\\
  v({s}_j) &\leftarrow & v({{s}_{j+1}})+\tilde{R}(s_{j+1}|s_j,a_j),\\
  N(s_j,a_j) &\leftarrow &  N(s_j,a_j)+1,
\end{eqnarray*}
where $(s_j,a_j)$ with $j< L$ is the edge in the forward trace.

\begin{figure}
\includegraphics[width=2.7in]{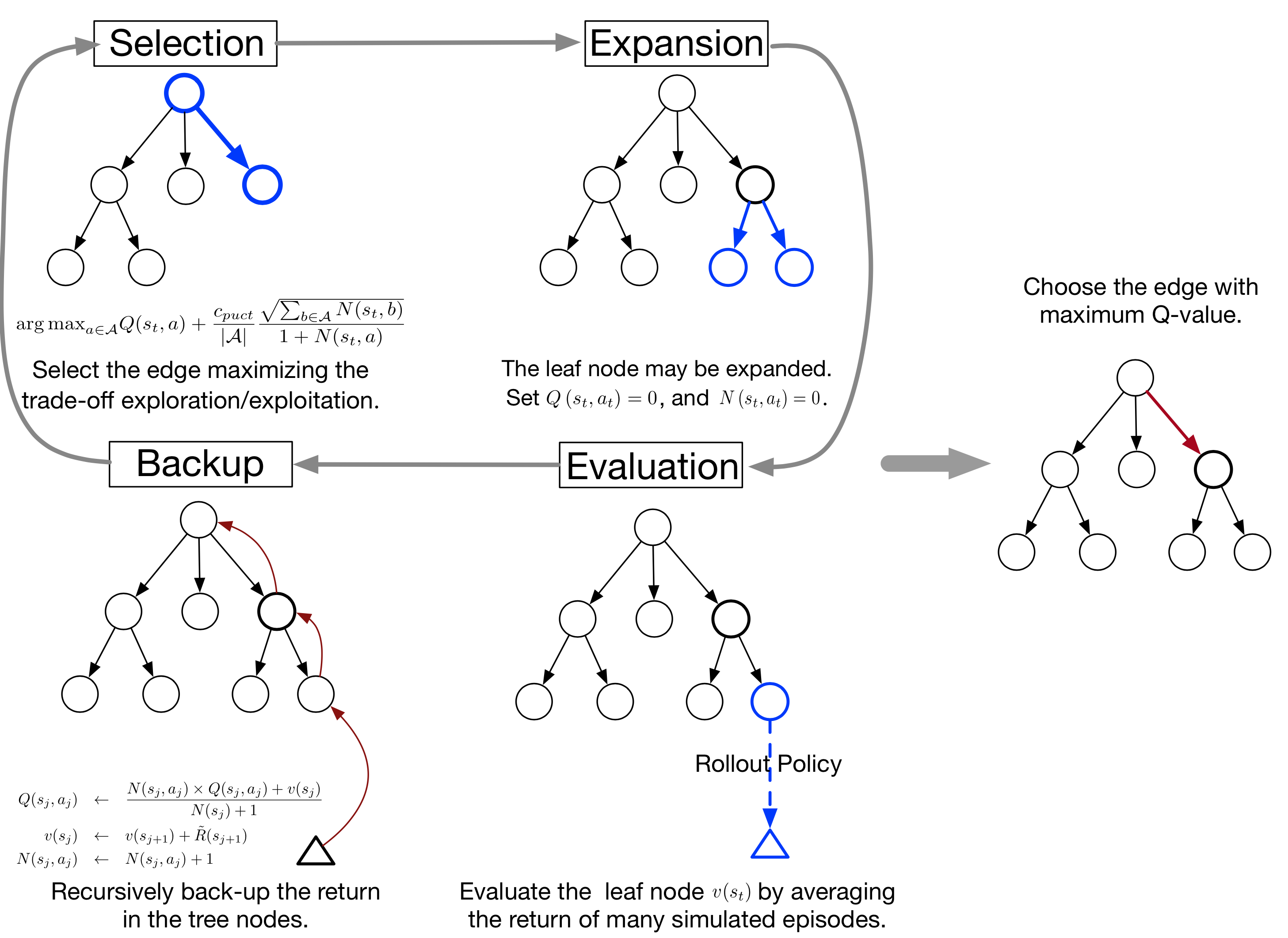}
    \caption{An illustration of MCTS for policy generation.}
    \label{fig:mcts} 
\end{figure}

This cycle of selection, evaluation, expansion and backup is repeated until the maximum iteration number has been reached. At this point, the best action has been chosen by selecting the action that leads to the highest reward (max child) as follows 
\begin{eqnarray}\label{equ:mcts_action}
    \pi(s_{t}) = \arg\max_{a\in\mathcal{A}} \tilde{Q}(s_t,a),
\end{eqnarray}
where $\pi(s_{t})$ is the policy generated by MCTS. To be noticed, the variance-bonus reward $\tilde{R}$ can be replaced with the $R$ for generating the best-known policy since it maximizes the expected reward based on the posterior so far.

\begin{algorithm}[tbh]
\small
\SetAlgoLined
\KwIn{Hype-parameters $\beta$, maximum trials number $N$.}
\KwOut{The malaria control policy $\bm{\pi}$.}
{\color{blue}\# Initialization.} \\
Generate random policy $\bm{\pi} = \left[a_1,\dots,a_T\right]$.\\
Apply policy $\bm{\pi}$ in real world and collect training samples $\mathcal{D}=\{(s_i,a_i,s_{i+1})\}_{i=1}^{T}$.\\
Update GP world model with $\mathcal{D}$.\\
Set number of trials $k=1$. \\
\For{$k<=N$}{
    Initialize time step $t=1$, and state $s_t$.\\
    \For{$t<=T$}{
    {\color{blue}\# Data Collection.} \\
    Generate action $a_t =\arg\max_{a\in\mathcal{A}} Q(s_t,a) $ by MCTS with \textbf{variance-bonus reward} $\tilde{R}$. \\ 
    Apply $a_t$ in real world and collect $s_{t+1}$.\\
    Update training samples $\mathcal{D} = \mathcal{D} \cup \{(s_t,a_t,s_{t+1})\}$.\\
    {\color{blue}\# Model Training.}\\
    Update GP world model with $\mathcal{D}$.\\
    Update $s_t=s_{t+1}$, and $t=t+1$.\\
    }
    Increase the training trials $k=k+1$.
}
{\color{blue}\# Generate final decision.}\\ 
Initialize time step $t=1$, state $s_t$, and policy $\bm{\pi}=\left[\right]$.\\
\For{$t<=T$}{
Generate action $a_t =\arg\max_{a\in\mathcal{A}} Q(s_t,a) $ by MCTS with \textbf{mean reward} $R$. \\ 
Predict $s_{t+1} = \mathbb{E}[s_{t+1}|s_t,a_{t}]$ with Gaussian Process.\\
Update $s_t=s_{t+1}$, and $t=t+1$.\\
Append action to final decision $\bm{\pi} = \bm{\pi}\oplus a_t$.
}
Return the malaria control policy $\bm{\pi}$.
\caption{Variance-Bonus Monte Carlo Tree Search}
\label{alg:training}
\end{algorithm}

Finally, we implement an iterative training procedure, as shown in Algorithm \ref{alg:training}, where we specify the order in which each component occurs within the iteration.

\subsection{Sample Complexity}
In the following, we derive sample complexity (\eg the required samples to learn near-optimal performances) for our proposed solution. In the worst case, when the reward function is equal everywhere, and all state-action pairs will be equally explored, VB-MCTS will have the same complexity bounds as over-exploration methods, such as GP-Rmax ~\cite{grande2014sample}. In practice, VB-MCTS will learn the optimal policy in fewer steps than the over-exploration methods because the high uncertainty region with low reward is not worth exploration. Following the general PAC-MDP theorem, Theorem 10 in~\cite{strehl2009reinforcement}, we derive the polynomial sample complexity of VB-MCTS.

\begin{theorem}\label{th:sample_complexity}
Assume that the feature space $\mathcal{X}\subset\mathbb{R}^{m}$ of state-action pairs is a compact domain and the target value is bounded $y\in [0, V_m]$. The reward and action value are Lipschitz continuous w.r.t the state-action pairs and let $L_q$ and $L_r$ be the Lipschitz constants for reward and action value respectively. If VB-MCTS is executed with $\beta_1 = \frac{L_r}{2 \sigma_{n}^{2}}$ and $\beta_2 = \frac{L_p}{2 \sigma_{n}^{2}}$ for any MDP $M$, then with probability $1-\delta$, VB-MCTS will follow a $4\epsilon$-optimal policy from its current state on all but
\begin{eqnarray}
    O\left(\frac{V_{m} \zeta(\epsilon, \delta)}{\epsilon(1-\gamma)} \log \frac{1}{\delta} \log \frac{1}{\epsilon(1-\gamma)}\right)
\end{eqnarray}
timesteps, with probability at least $1 - 2\delta$, where
$\zeta(\epsilon, \delta)=\left(\frac{4 V_{m}^{2}}{\epsilon^{2}(1-\gamma)^{2}} \log \left(\frac{2\left|\mathcal{N}_{\mathcal{X}}\left(\tau\left(\sigma_{\text {tol }}^{2}\right)\right)\right|}{\delta}\right)\right) \mathcal{N}_{\mathcal{X}}\left(\tau\left(\sigma_{\text {tol }}^{2}\right)\right)$,
$ \sigma_{\text{tol}}^{2} = \frac{2 \omega_{n}^{2} \epsilon_{1}^{2}}{V_{m}^{2} \log \left(\frac{2}{\delta_{1}}\right)}$. $\mathcal{N}_{\mathcal{X}}(\tau)$ is the covering number of domain $\mathcal{X}$, which defines the cardinality of the minimal set $C=\left\{c_{i}, \ldots, c_{N_{c}}\right\} \text { s.t. } \forall x \in \mathcal{X}, \exists c_{j} \in C \text { s.t. } d\left(x, c_{j}\right) \leq \tau,$ where $d(\cdot,\cdot)$ is a distance measure.
\end{theorem}

\begin{proof}[Sketch of Proof]
The proof is organized by showing the required key properties (optimism, accuracy, and learning complexity) for general PAC-MDP~\cite{strehl2009reinforcement} are satisfied.
 
\paragraph{Optimism} Assume that MCTS can find the optimal decision under every state $s$ with enough computation resources. Then, $\forall (s_t,a_t)$, the optimal action value $Q^\ast(s_t,a_t)$ satisfies
% {\small
% \begin{eqnarray*}
%      && Q^\ast(s_t,a_t) = {\arg\max}_{a_{t+1}\in\mathcal{A}}R(s_{t+1}) + \gamma Q^\ast(s_{t+1},a_{t+1}) \\
%     &\leq & R(s'_{t+1}) + \frac{L_r}{\omega^2_n}\text{Var}_{s'_{t+1}\sim GP(s_t,a_t)}(R(s'_{t+1})) \\
%     &+& \gamma \tilde{Q}^\ast(s'_{t+1},a_{t+1}) + \gamma\frac{L_q}{\omega^2_n}\text{Var}_{s’_{t+1}\sim GP(s_t,a_t)}(s'_{t+1})  \\
%     &\leq& {\arg\max}_{a'_{t+1}\in\mathcal{A}} R(s'_{t+1}) + \frac{L_r}{\omega^2_n}\text{Var}_{s'_{t+1}\sim GP(s_t,a_t)}(R(s'_{t+1})) \\
%     &+& \gamma \tilde{Q}^\ast(s'_{t+1},a'_{t+1}) + \gamma\frac{L_q}{\omega^2_n}\text{Var}_{s’_{t+1}\sim GP(s_t,a_t)}(s'_{t+1})\\
%     &=&\tilde{Q}(s_t,a_t),
% \end{eqnarray*}
% }
{\small
\begin{equation*}
\begin{aligned}
& Q^{*}\left(s_{t}, a_{t}\right)=\arg \max _{a_{t+1} \in \mathcal{A}} R\left(s_{t+1}\right)+\gamma Q^{*}\left(s_{t+1}, a_{t+1}\right) \\
\leq & R\left(s_{t+1}^{\prime}\right)+\frac{L_{r}}{\omega_{n}^{2}} \operatorname{Var}_{s_{t+1}^{\prime}\sim G P\left(s_{t}, a_{t}\right)} \left[R\left(s_{t+1}^{\prime}\right)\right] \\
+& \gamma \tilde{Q}^{*}\left(s_{t+1}^{\prime}, a_{t+1}\right)+\gamma \frac{L_{q}}{\omega_{n}^{2}} \operatorname{Var}_{s_{t+1}^{\prime}\sim G P\left(s_{t}, a_{t}\right)} \left[s_{t+1}^{\prime}\right] \\
\leq & \arg \max _{a_{t+1}^{\prime} \in \mathcal{A}} R\left(s_{t+1}^{\prime}\right)+\frac{L_{r}}{\omega_{n}^{2}} \operatorname{Var}_{s_{t+1}^{\prime}\sim G P\left(s_{t}, a_{t}\right)} \left[R\left(s_{t+1}^{\prime}\right)\right] \\
+& \gamma \tilde{Q}^{*}\left(s_{t+1}^{\prime}, a_{t+1}^{\prime}\right)+\gamma \frac{L_{q}}{\omega_{n}^{2}} \operatorname{Var}_{s_{t+1}^{\prime}\sim G P\left(s_{t}, a_{t}\right)} \left[s_{t+1}^{\prime}\right] \\
=& \tilde{Q}\left(s_{t}, a_{t}\right),
\end{aligned} 
\end{equation*}
}
where $s_{t+1} {=} P(s_t,a_t)$, $s'_{t+1} {=}\mathbb{E}_{s_{t+1}\sim
GP(s_t,a_t)}[s_{t+1}|s_t,a_t]$. 
The maximum error of propagating $Q(s_{t+1},a_{t+1})$ instead of $Q(s'_{t+1},a_{t+1})$ is given by $\|s'_{t+1}-s_{t+1}\|L_q$. 
The norm $\|s'_{t+1}-s_{t+1}\|$ is upper bounded by the regularization error
$\frac{L_{q}}{\omega_{n}^{2}} \operatorname{Var}_{s_{t+1}^{\prime}\sim G P\left(s_{t}, a_{t}\right)} \left[s_{t+1}^{\prime}\right]$ 
with probability of $(1-\delta)$~(Section A.2 in ~\cite{grande2014computationally}). 
In the same way, we have  $R(s'_{t+1}) + \frac{L_r}{\omega^2_n}\text{Var}_{s'_{t+1}\sim GP(s_t,a_t)}[R(s'_{t+1})] - R(s_{t+1})\geq 0$ with probability $1-\delta$. 
%When the transition function is {\it "unknown"}~\cite{pazis2013pac}, $\tilde{Q}(s_t,a_t)\geq Q^\ast(s_t,a_t)$ is true with probability $1-\delta$. 
It follows that Algorithm \ref{alg:training} remains optimistic with probability at least 1-2$\delta$.

\paragraph{Accuracy} Following the Lemma 1 from \cite{grande2014sample}, the prediction error at $x_i\in \mathcal{X}$ is bounded in probability $Pr\left\{\left|\hat{\mu}\left(x_{i}\right)-f\left(x_{i}\right)\right| \geq \epsilon_{1}\right\} \leq \delta_{1}$ if the predictive variance of the GP at $ \sigma_n^{2}\left(x_{i}\right) \leq \sigma_{\text{tol}}^{2}=\frac{2 \omega_{n}^{2} \epsilon_{1}^{2}}{V_{m}^{2} \log \left(\frac{2}{\delta_{1}}\right)}$. 

\paragraph{Learning complexity} \textbf{(a)} For $\forall x\in\mathcal{X}$ with $\sigma_0^2 = 1$, let $\rho = \min\limits_{j=1,\dots,n} k(x,x_j)$, $d_{\max} = \max\limits_{j=1,\dots,n} d(x,x_j)$, then $d_{\max}$ can be rewritten as $d_{\max} = \sqrt{2\theta^2\log{\frac{1}{\rho}}}$ with $\rho =e^{-\frac{d_{\max}^2}{2\theta^2}}$. If $n\geq \frac{\omega_n^2}{\sigma_{\text{tol}}^2-1+\rho^2}$, the posterior variance of $x$ satisfies $\sigma_n^2(x)\leq 1-\frac{n\rho^2}{n+\omega^2_n}\leq\frac{n(1-p^2)+\omega^2_n}{n}\leq \sigma^2_{\text{tol}}$. \textbf{(b)} Since $\mathcal{X}$ is a compact domain with the length of $i$-th $\left(i\in [1,m]\right)$ dimension as $L_i$, we have $\mathcal{N}_{\mathcal{X}}\left(d_{\max }\right) \leq \mathcal{N}_{\mathcal{X}}\left(\frac{d_{\max }}{2}\right)$ and $\mathcal{X}$ can be covered by $2^{m} \prod_{j=1}^{m} L_{j} / d_{\max }^{m}$ balls $\left\{B\left(c_{j}, \frac{d_{\max }}{2}\right)\right\}_{j}$, which centers at $c_j$ with a radius of $\frac{d_{\max }}{2}$.\\

Given \textbf{(a)} and \textbf{(b)}, $\forall x\in \mathcal{X}$, $\exists c_j\in\mathcal{X}$ s.t. $B(x,d_{\max})\supset B(x,d_{\max}/2)$. If there are at least $\frac{\omega_n^2}{\sigma_{\text{tol}}^2-1+\rho^2}$ observations in $B(x,d_{\max}/2)$, then (2) satisfied. Combining the Lemma 8 from~\cite{strehl2009reinforcement} and denoting $\tau(\sigma^2_{\text{tol}}) = {d_{\max}}/2$, the total number of updates occurs will be bounded by $\zeta(\epsilon, \delta)=\left(\frac{4 V_{m}^{2}}{\epsilon^{2}(1-\gamma)^{2}} \log \left(\frac{2\left|\mathcal{N}_{\mathcal{X}}\left(\tau\left(\sigma_{\mathrm{tol}}^{2}\right)\right)\right|}{\delta}\right)\right) \mathcal{N}_{\mathcal{X}}\left(\tau\left(\sigma_{\mathrm{tol}}^{2}\right)\right)$ with probability of $1-\delta$.  

Now that the key properties of optimism, accuracy, and learning complexity have been established, the general PAC-MDP Theorem 10 of~\cite{strehl2009reinforcement} is invoked.
\end{proof}

Theorem~\ref{th:sample_complexity} allows us to guarantee that the number of steps in which the performance of VB-MCTS is significantly worse than that of an optimal policy starting from the current state is at most log-linear in the covering number of the state-action space with probability $1-\delta$.

\section{Experiments}
In this section, we conduct extensive experiments on two different OpenMalaria~\cite{smith2008towards} based simulators: {\it SeqDecChallenge} and {\it ProveChallenge}, which are the testing environments used in KDD Cup 2019. In these simulators, the parameters of the simulator are hidden from the RL agents since the {\it“simulation parameters”} for SSA are unknown for policymakers. Additionally, to simulate the real disease control problem, only {\bf 20 trials} were allowed before generating the final decision, which is much more challenging than traditional RL tasks. These simulating environments are available at \url{https://github.com/IBM/ushiriki-policy-engine-library}.

% {\color{red} In this section, we conduct extensive experiments on the simulated environments provided in 2019 KDD Cup Humanity RL Track$^1$ to demonstrate the advantage of our proposed method. In this challenge, three different OpenMalaria based simulators --- {\it SeqDecChallenge}, {\it ProveChallenge} and {\it FinalChallenge} --- were provided for policy searching in {\bf 20 trials}. The source code for simulating environment is available at \url{https://github.com/slremy/netsapi}.}

\paragraph{Agents for Comparison}
To show the advantage of VB-MCTS, many benchmarking reinforcement learning methods and open source solutions have been deployed to verify its effectiveness: \textbf{Random Policy}: The random policy is executed in 20 trials and chooses the generated policy with the maximum reward as the final decision. \textbf{SMAB}: This kind of policy treats the problem as a Stochastic Multi-Armed Bandit problem and independently optimizes the policy every year with Thompson sampling~\cite{chapelle2011empirical}. \textbf{CEM}: Cross-Entropy Method is a simple gradient-free policy searching method~\cite{szita2006learning}. \textbf{CMA-ES}: CMA-ES is a gradient-free evolutionary approach to optimizing non-convex objective functions~\cite{krause2016cma}. \textbf{Q-learning-GA}: It learns the malaria control policy by combining Q-learning and Genetic Algorithm. \textbf{Expected-Sarsa}: It collects 13 random episodes and runs expected value SARSA~\cite{van2009theoretical} for 7 episodes to improve the best policy using the collected statistics. \textbf{GP-Rmax}: It uses GP learners to model $T$ and $R$ and replaces the value of any $Q^\ast (s, a)$ where $T (s, a)$ is {\it "unknown"} with a value of $\frac{R_{\max }}{1-\gamma}$~\cite{li2011knows,grande2014sample}. \textbf{GP-MC}: It employs Gaussian Process to regress the world model. The policy is generated by sampling from the posterior and choosing the max rewarded action. \textbf{VB-MCTS}: Our proposed method.

\paragraph{Implementation Details} We build a feature map of 14-dimension for this task, which includes the periodic feature and cross term feature. Specifically, the feature map is set as,
{\small\begin{eqnarray}
\nonumber
&&\phi(s_t,a_t) = [t,t\%2,t\%3,r_{t-1},a_{(t-1,{ITN})},a_{(t-1,{IRS})},\\ \nonumber
&& a_{(t,{ITN})},a_{(t,{IRS})},a_{(t,{ITN})}\times a_{(t,{IRS})},\\ \nonumber
&& a_{(t-1,{ITN})}\times a_{(t-1,{IRS})},a_{(t,{ITN})}\times a_{(t-1,{ITN})},  \\\nonumber
&& a_{(t,{IRS})}\times a_{(t-1,{IRS})},a_{(t,{ITN})}\times \left(1-a_{(t-1,{ITN})}\right),  \\ \nonumber
&& a_{(t,{IRS})}\times \left(1-a_{(t-1,{IRS})}\right)],
\end{eqnarray}}
where $t,t\%2,t\%3$ are the periodic features and $a_{(\ast,{\ast})}\times a_{(\ast,{\ast})}$ is the cross term feature. Since the predicted variance of state and reward are the same in our setting, we empirically set the sum of exploration/exploitation parameters as $\beta_1+\beta_2 = 3.5$ in the experiments. In MCTS, $c_{puct}$ is 5, and only the top 50 rewarded child nodes are expanded. The number of iterations does not exceed 100,000. For Gaussian Process, to avoid overfitting problems, 5-fold cross-validation is performed during the updates of the GP world model. Particularly, we use 1-fold for training and 4-fold for validation, which ensures the generalizability of the GP world model over different state-action pairs. Our implementation and all baseline codes are available at \url{https://github.com/zoulixin93/VB_MCTS}.

{\centering
\small
\begin{table}[!hpt]
\tabcolsep 0.04in
\caption{Performance comparisons between different agents.}
\label{tab:results}
\begin{tabular}{l|| c c c} \ChangeRT{1pt}
\centering
{\multirow{3}*{Agents}}
&\multicolumn{3}{c}{\textit{SeqDecChallenge}} \\
& Med. Reward & Max Reward & Min Reward  \\ \hline\hline
Random Policy & 167.79 & 193.24 & 135.06  \\
SMAB & 209.05 & 386.28 & -6.44 \\
\hline
CEM & 179.30 & 214.87 & 120.92 \\
CMA-ES & 185.34 & 246.12 & 108.18 \\
Q-Learning-GA & 247.75 & 332.40 & 171.33  \\
Expected-Sarsa & 462.76 & 495.03 & 423.93  \\
\hline
GP-Rmax & 233.95 & 292.99 & 200.35 \\
GP-MC & 475.99 & 499.60 & 435.51 \\
\rowcolor{lightgray}
VB-MCTS & \textbf{533.38} & \textbf{552.78} & \textbf{519.61}  \\
\ChangeRT{0.8pt}
{\multirow{3}*{Agents}} 
&\multicolumn{3}{c}{\textit{ProveChallenge}} \\
& Med. Reward & Max Reward & Min Reward  \\ \hline\hline
Random Policy  & 248.25 & 464.92 & 55.24  \\
SMAB & 18.02 & 135.37 & -56.86 \\
\hline
CEM & 229.61 & 373.83 & 20.09 \\
CMA-ES & 289.03 & 314.57 & 92.95 \\
Q-Learning-GA & 242.97 & 325.24 & 88.70\\
Expected-Sarsa & 190.08 & 296.16 & 140.86 \\
\hline
GP-Rmax & 287.45 & 371.49 & 153.98 \\
GP-MC  & 300.37 & 447.15 & \textbf{263.96}\\
\rowcolor{lightgray}
VB-MCTS & \textbf{352.17} & \textbf{492.23} & 259.97 \\
\ChangeRT{0.8pt}
\end{tabular}
\end{table}}

\subsection{Results}
\paragraph{Main Results} In Table \ref{tab:results}, we report the median reward, the maximum reward, and the minimal reward over 10 independent repeat runs. The results are quite consistent with our intuition. We have the following observations: \textbf{(1)} For the finite-horizon decision-making problem, treating it as SMAB does not work, and the delayed influence of actions can not be ignored in malaria control. As presented in Table~\ref{tab:results}, SMAB's performances have large variance and are even worse than random policy in {\it ProveChallenge}. \textbf{(2)} Overall, two model-based methods~(GP-MC and VB-MCTS) consistently outperform the model-free methods~(CEM, CMA-ES, Q-learning-GA, and Expected-Sarsa) in {\it SeqDecChallenge} and {\it ProveChallenge}, which indicates that empirically model-based solutions are generally more data-efficient than model-free solutions. From the results in Table~\ref{tab:results}, performances of model-free methods are defeated by model-based methods with a large margin in {\it SeqDecChallenge}, and their performances are almost the same as the random policy in {\it ProveChallenge}. \textbf{(3)} The proposed VB-MCTS can outperform all the baselines in {\it SeqDecChallenge} and {\it ProveChallenge}. Compared with GP-MC and GP-Rmax, the advantage from the efficient MCTS with variance-bonus reward leads to the success on {\it SeqDecChallenge} and {\it ProveChallenge}.
	
{\begin{figure}[hbt]
\centering
\includegraphics[width=2.7in]{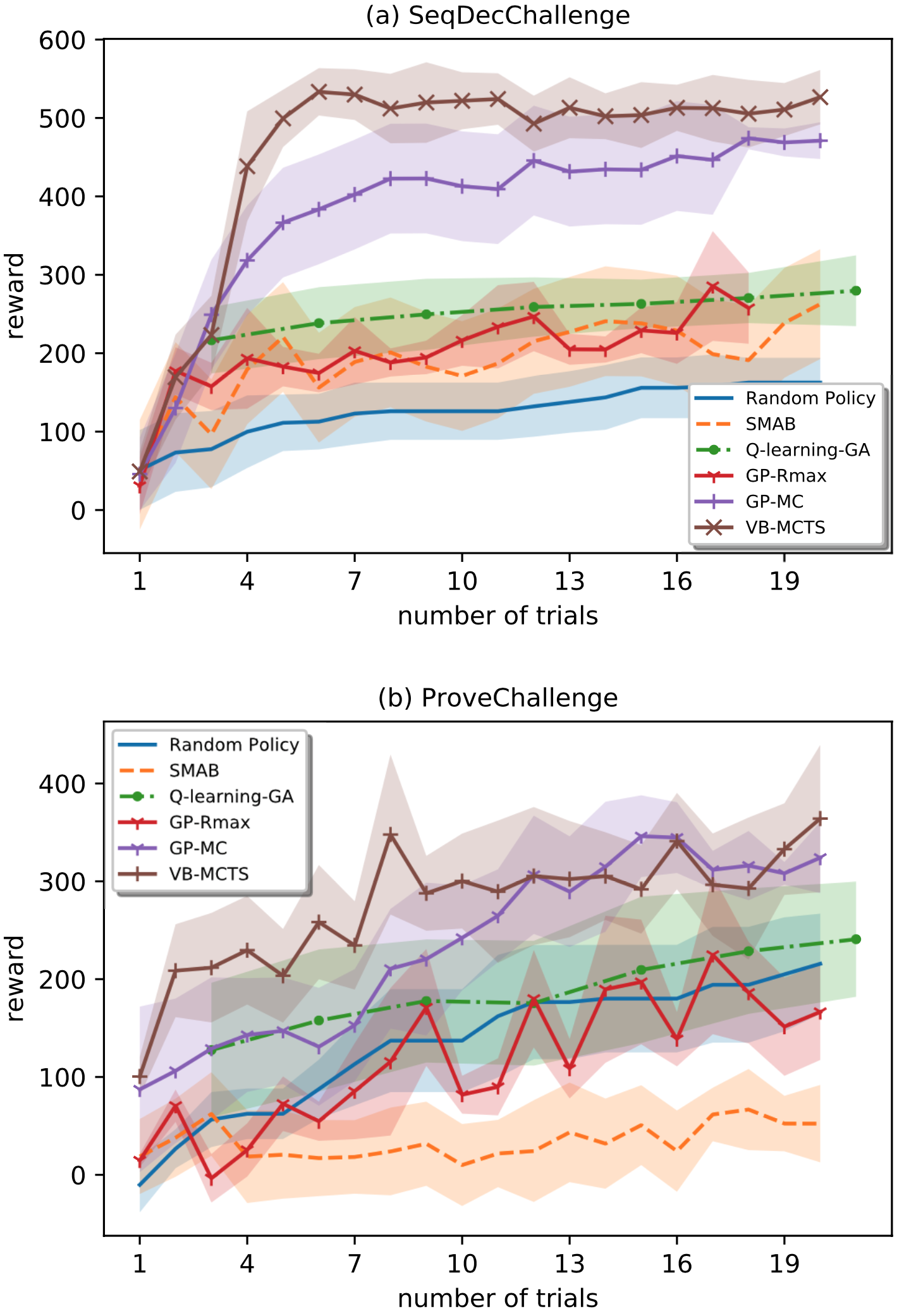}
\caption{Malaria policy learning curves in five different agents.}
\label{fig:learning_curves}
\end{figure}}

\paragraph{Data Efficiency}
This paragraph compares data-efficiency~(required trials) of VB-MCTS with other RL methods that learn malaria policies from scratch. In Figure~\ref{fig:learning_curves}(a) and~\ref{fig:learning_curves}(b), we report agents' performances after collecting every trial episode in {\it SeqDecChallenge} and {\it ProveChallenge}. The horizontal axis indicates the number of trials. The vertical axis shows the average performance after collecting every episode. Figure \ref{fig:learning_curves}(a) and \ref{fig:learning_curves}(b) highlighted that our proposed VB-MCTS approach (brown) requires on average only 8 trials to achieve the best performances in {\it SeqDecChallenge} and {\it ProveChallenge}, including the first random trial. The results indicate that VB-MCTS can outperform the state-of-the-art method on both data efficiency and performance.

\section{Conclusion and Future Work}

We proposed a model-based approach employing the Gaussian Process to regress the state transition for data-efficient RL in malaria control. By planning with variance-bonus reward, our method can naturally deal with the dilemma of exploration and exploitation by efficiently MCTS planning. Extensive experiments conducted on the challenging malaria control task have demonstrated the advantage of VB-MCTS over state-of-the-arts both on performance and efficiency. However, the stationary setting of MDP may be unrealistic due to the development of disease control tools and the evolution of the disease. Therefore, data-efficient reinforcement learning under nonstationary settings will be more realistic and more challenging task.

\bibliographystyle{named}
\bibliography{ijcai21}

\end{document}